\documentclass[letterpaper, 10 pt, conference]{ieeeconf}  % Comment this line out
\IEEEoverridecommandlockouts                              % This command is only
\usepackage{bbm}
\usepackage{mathrsfs}
\usepackage{verbatim}
\usepackage{amsmath}
\usepackage{amsfonts}
\usepackage{bm}
\usepackage{graphicx}
\usepackage{float}

\usepackage{amsthm}
\usepackage{xcolor}
\usepackage[colorinlistoftodos]{todonotes}
\usepackage{mathtools}
\usepackage{cite}
\usepackage{colortbl}
\usepackage{subcaption}
\usepackage{hyperref}

\theoremstyle{definition}
\newtheorem{assumption}{Assumption}

\newtheorem{remark}{Remark}
\newtheorem{problem}{Problem}

\theoremstyle{plain}

\newtheorem{theorem}{Theorem}
\newtheorem{lemma}{Lemma}

\title{\LARGE \bf Off-Policy Evaluation for Sequential Persuasion Process\\with Unobserved Confounding}

\author{Nishanth Venkatesh S.$^{1}$, {\itshape{Student Member, IEEE}}, Heeseung Bang$^{2}$, {\itshape{Member, IEEE}},\\and Andreas A. Malikopoulos$^{1,2}$, {\itshape{Senior Member, IEEE}}
	\thanks{This research was supported by in part by NSF under Grants CNS-2401007, CMMI-2348381, IIS-2415478 and in part by MathWorks.}
    \thanks{$^{1}$Department of Systems Engineering, Cornell University, Ithaca, NY 14850 USA.}
    \thanks{$^{2}$School of Civil and Environmental Engineering, Cornell University, Ithaca, NY 14853 USA (email: \texttt{ns942@cornell.edu; h.bang@cornell.edu; amaliko@cornell.edu).}} }

\begin{document}

\maketitle
\thispagestyle{empty}

\begin{abstract}
In this paper, we expand the Bayesian persuasion framework to account for unobserved confounding variables in sender-receiver interactions. While traditional models typically assume that belief updates follow Bayesian principles, real-world scenarios often involve hidden variables that impact the receiver’s belief formation and decision-making. We conceptualize this as a sequential decision-making problem, where the sender and receiver interact over multiple rounds. In each round, the sender communicates with the receiver, who also interacts with the environment. Crucially, the receiver’s belief update is affected by an unobserved confounding variable. By reformulating this scenario as a Partially Observable Markov Decision Process (POMDP), we capture the sender’s incomplete information regarding both the dynamics of the receiver’s beliefs and the unobserved confounder. We prove that finding an optimal observation-based policy in this POMDP is equivalent to solving for an optimal signaling strategy in the original persuasion framework. Furthermore, we demonstrate how this reformulation facilitates the application of proximal learning for off-policy evaluation (OPE) in the persuasion process. This advancement enables the sender to evaluate alternative signaling strategies using only observational data from a behavioral policy, thus eliminating the necessity for costly new experiments.

\end{abstract}

\section{Introduction}
\label{sec:intro}

Strategic information sharing plays a critical role in economic interactions, policy design, and multi-agent systems\cite{Dave2020SocialMedia,Malikopoulos2021,venkatesh2023connected}.
Bayesian persuasion (BP) was first introduced by Kamenica and Gentzkow \cite{kamenica2011bayesian} as a powerful framework for analyzing how a sender can strategically reveal information to influence a receiver's decisions. In the standard setting, a sender commits to an information disclosure policy before observing the state of the world, and the receiver, after observing the sender's message, forms posterior beliefs and takes an action that affects both the sender's and the receiver's utilities.

Despite its theoretical elegance, BP rests on assumptions that may not hold in practical settings.
First, the framework presupposes that the sender possesses complete information about the receiver, including their observation process and all features that influence their decision-making (including utility functions).
Second, it assumes that the receiver has sufficient knowledge about the underlying state space to make Bayesian inferences based on the sender's signaling policy and the shared signal.
Many research efforts have explored various directions to relax these restrictive assumptions \cite{renault2017optimal,farhadi2022dynamic,massicot2025almost} or to resolve computational challenges \cite{sayin2021bayesian}.
For example, Castiglioni et al. \cite{castiglioni2020online} introduced online BP to address scenarios where the sender lacks knowledge of the receiver's utility function, proposing an online learning approach to acquire this information iteratively. Several works have extended the framework to sequential decision-making processes.
Gan et al. \cite{gan2022bayesian} and Wu et al. \cite{wu2022markov} formulated BP within Markov decision processes, where a sender engages in multiple rounds of interaction with different myopic receivers at each timestep.
Building on this foundation, Bacchiochchi et al. \cite{bacchiocchi2024markov} leveraged reinforcement learning techniques for settings where the sender has limited prior knowledge of the environment. In a related vein, Lin et al. \cite{lin2023information} developed the Markov signaling game framework and derived signaling gradients to facilitate reinforcement learning approaches.

% 16, 45, 17, 29

Although these efforts provide different approaches to relaxing the assumptions, some limitations remain: they generally assume that all relevant variables affecting the receiver's belief formation are observable to the sender. In practice, unobserved confounding variables often influence how receivers interpret information and make decisions. For instance, a user's response to recommendations may depend on contextual factors unknown to the recommender system, or a policymaker's reaction to economic data may be affected by unobserved political constraints. Such confounding introduces a significant challenge to the design of effective persuasion strategies.

In this paper, we address this gap by extending the Bayesian persuasion framework to account for unobserved confounding in sender-receiver interactions. First, we introduce a formulation that models sequential persuasion with unobserved confounding as a Partially Observable Markov Decision Process (POMDP). This formulation captures the sender's incomplete information about both the receiver's belief dynamics and the unobserved confounder. We prove that finding an optimal observation-based policy in this POMDP is equivalent to solving for an optimal signaling strategy in the original persuasion framework.
Second, we demonstrate how this reformulation enables the application of proximal learning, a causal inference technique  
introduced by Miao et al.\cite{miao2018identifying} and extended further to POMDP settings by \cite{tennenholtz2020off}.
Using proximal learning, we address the issue of unobserved confounding for off-policy evaluation (OPE) in the persuasion process.
This allows the sender to assess alternative signaling strategies using only observational data from a behavioral policy without requiring costly new experiments.
The capability for OPE serves as a foundation for reinforcement learning through policy optimization.
Our approach provides a principled way to design robust persuasion strategies in the presence of confounding, with applications in recommendation systems, strategic communications, and human-AI interaction\cite{faros2023adherence,dave2024airecommend}.

% \Nish{One part of the intro should be about Bayesian Persuasion. Who introduced it and what has been done along these lines and any applications even as simple examples that were implemented in papers.
% As far as I see I can find regret minimization, etc all in the online setting. We should mention that we address the offline setting or take a step towards offline setting. (Later we can leverage knowledge from offline and do a Hybrid version)}

% \Nish{Since we are addressing audience from controls, it would make sense to mention what off-policy evaluation is and why do we even need this. Give references from typical OPE along with OPE/RL under confounding as a part of intro.}

% %%% Contribution %%%
% 1) We define and formulate a problem where, in a sequential persuasion process, how unobserved confounders play a role.
% 2) Formulate an POMDP to model the same framework.
% 3) Show equivalence of POMDP v/s Persuasion framework.
% 4) OPE for persuasion framework via POMDP using proximal learning.

%%% Organization %%%
The remainder of this paper is organized as follows. In Section \ref{sec:BP}, we review the standard Bayesian persuasion framework, followed by a visualization of the causal graph of this framework. 
In Section \ref{sec:spp}, we present our extension of the persuasion framework, formulating a sequential decision-making problem for the sender, and highlights the role of hidden confounders.
In Section \ref{sec:pomdp}, we formulate a POMDP, equivalent to the persuasion framework.
In Section \ref{sec:OPE}, we demonstrate how proximal learning can be applied to the POMDP to enable OPE for the persuasion framework. 
In Section \ref{Application}, we present an application that can be modeled using the persuasion framework.
Finally, in Section \ref{sec:conclusion}, we draw concluding remarks and highlight some future directions.

\section{Bayesian Persuasion and Causality} \label{sec:BP}

% In this paper, we consider the problem of Bayesian persuasion under unobserved confounding. We first introduce general Bayesian persuasion framework, followed by a causal description based on causal graphs. 
% Next, we will explain how a confounder could affect the interaction between sender and receiver in the persuasion setting before presenting an analysis of how the effects of a confounder be adjusted.
% \textcolor{blue}{I am wondering if I need to explain the regular Bayesian persuasion before moving into our framework. I'll have it here for now. If there are space constraints we can shorten this.}
%%%%%%%%%%%%%%%%%%%%%%%%%%%%

% \subsection{Bayesian Persuasion}
Bayesian persuasion illustrated in Fig. \ref{fig:BP}, models an interaction between two agents, where a sender communicates using signals with a receiver who interacts with the environment based on this communication.
%Generally, the sender and receiver are noncooperative and have different objectives.
%Given that the receiver is self-interested and the sender knows more about the environment, the goal of the sender is to select a communication strategy to influence the receiver's behavior in its favor.
%We illustrate the layout of this framework in Fig. \ref{fig:BP}.

\begin{figure}[t]
    \centering    \includegraphics[width=0.85\linewidth]{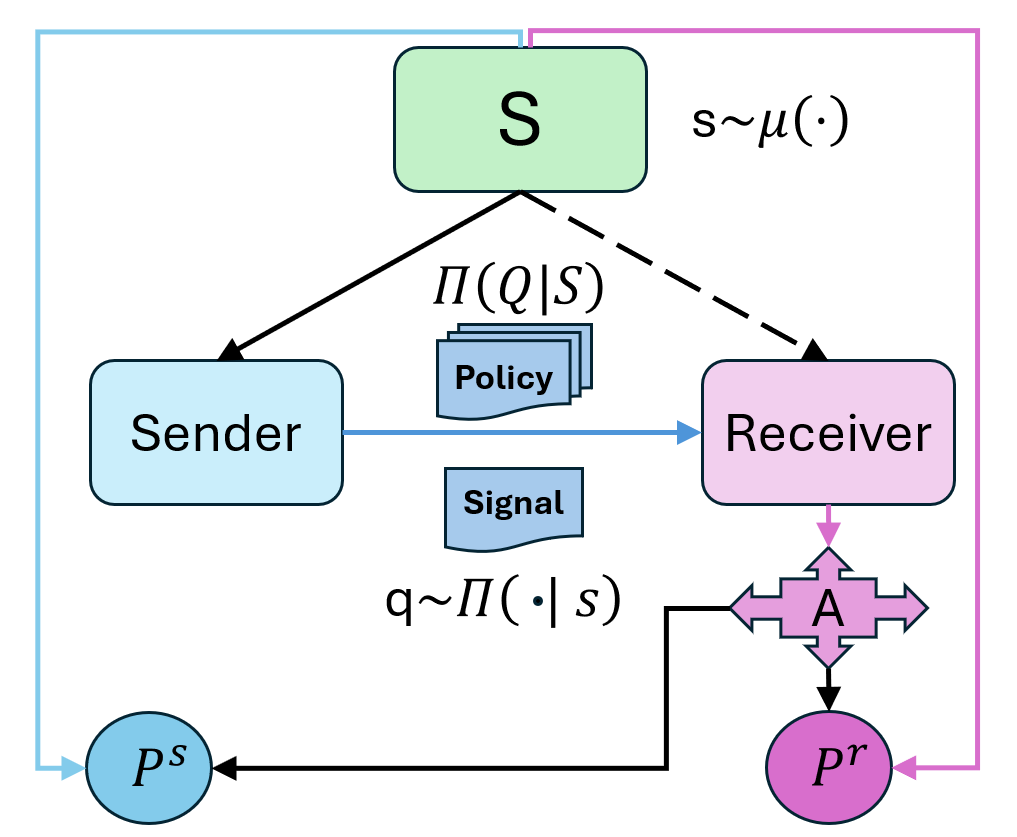}
    \caption{Bayesian persuasion scheme.}
    \label{fig:BP}
        \vspace{-15pt}
\end{figure}
Let the state of the environment be denoted by the random variable $S$, and let any realization $s$ belong to a finite set $\mathcal{S}$. 
The state is sampled from a distribution $\mu \in \Delta(\mathcal{S})$, also known as \textit{prior}.
In BP, both agents know the \textit{prior} $\mu$, while only the sender observes the realized state $s$.
The receiver interacts with the environment by taking actions, denoted by the random variable $A$, with realization $a$ taking values from a finite set $\mathcal{A}$, and this interaction gives a reward to each of the agents.
The sender and receiver obtain rewards $P^s$ and $P^r$ given by the function $\rho^s: \mathcal{S} \times \mathcal{A} \to [m^s,M^s]$ and $\rho^r: \mathcal{S} \times \mathcal{A} \to [m^r,M^r]$, respectively, where $m^s$, $M^s$, $m^r$, $M^r \in \mathbb{R}$ are the lower and upper bounds of their rewards.

In BP, the sender can select a signaling policy $\pi$ from a finite set of policies denoted by $\mathcal{P}$.
The sender then commits this signaling policy, $\pi : \mathcal{S} \to \Delta(\mathcal{Q})$, 
where $\mathcal{Q}$ is a finite space of signals.
During the interaction, the sender samples a signal from signaling policy, i.e., $q \sim \pi(\cdot|s)$, and sends it to the receiver.
%Since the signal $q$ encodes information about the state $s$, the signal influences the receiver's belief about the state $s$, which, in turn, will affect the self-interested action $a$.
Based on the signal, the receiver performs a Bayesian update of their belief about $S$ considering both the prior $\mu$ and the signaling policy $\pi$.
For each $s\in \mathcal{S}$, the resulting posterior distribution is given by
\begin{align}
    p^{\pi}(s|q) = \frac{\pi(q|s)\;\mu(s)}{\sum_{s' \in \mathcal{S}}\pi(q|s')\;\mu(s')},
\end{align}
where $p^{\pi}(s|q)$ is the posterior probability of any realization $s$ under the committed signaling policy $\pi$ and the received signal $q$.
The receiver then picks their optimal action $a^*(q;\pi)$ based on the posterior as follows:
\begin{align}
    a^*(q;\pi)= \arg\max_{a \in \mathcal{A}} \sum_{s\in\mathcal{S}} p^{\pi}(s|q)\; \rho^r(s,a).
\end{align}
Then, the probability of taking action $A$ given the signal $q$ becomes deterministic, that is, $p^{\pi}(a|q) = 1$ if $a=a^*(q;\pi)$ and $p^{\pi}(a|q) = 0$, otherwise.
We define the performance of a signaling policy $\pi$ for the sender as
%\begin{align}
%    J(\pi)=\;& \mathbb{E}^{\pi}[\rho^s(S,A)],\\
%    =\;& \sum_{s \in \mathcal{S}}\sum_{q \in  \mathcal{Q}}\; \mu(s)\; \pi(q|s)\; \rho^s(s,a^*(q;\pi)).
%\end{align}
\vspace{-2pt}
\begin{align}
    J(\pi)=\sum_{s \in \mathcal{S}}\sum_{q \in  \mathcal{Q}}\; \mu(s)\; \pi(q|s)\; \rho^s(s,a^*(q;\pi)).
\end{align}
The goal of the sender is to compute the best signaling policy $\pi^* = \arg\max_{\pi} J(\pi)$.

Next, we discuss the causal graph of the BP framework (see Fig. \ref{fig:causal_graph_BP}).
This graph is made up of nodes and directed edges.
Nodes denote random variables, which include the state, action, and rewards, which are the same quantities in the persuasion process. 
The posterior, which represents the belief of the receiver about the state of the system, is denoted by the variable $B$. 
The policy of the sender is indicated by $\Pi$.
The edges, also known as causal links, indicate the direction of causality.
Since the signal $Q$ is sampled from the sender's signaling policy conditioned on the state, there exists a causal link from $S$ and $\Pi$ to the node $Q$.
Furthermore, the belief update process is influenced by both the policy and the received signal, establishing causal links from these nodes to the belief $B$.
The posterior, in turn, determines the receiver's action $A$, leading to a directed edge from $B$ to $A$.
Finally, by definition, both the state $S$ and the receiver's action $A$, have direct causal links to the reward nodes $P^r$ and $P^s$.

% ADD MORE EXPLANATION 

\begin{figure}[t]
    \centering
    \includegraphics[width=0.8\linewidth]{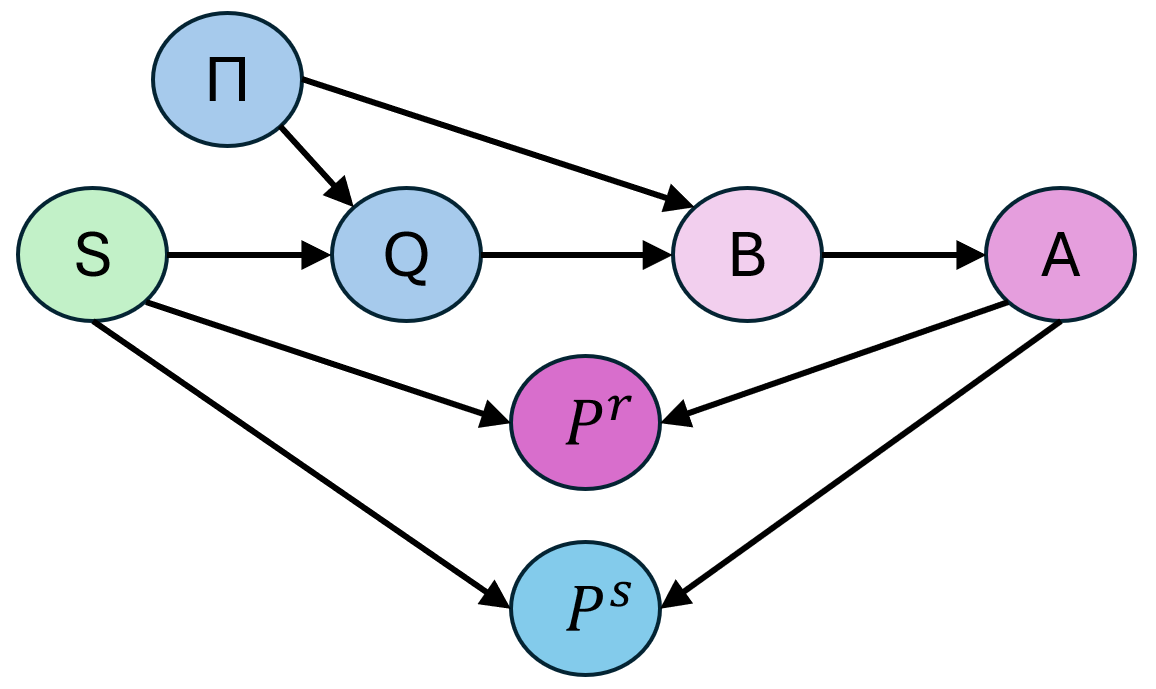}
    \caption{Causal graph - Bayesian persuasion.}
    \label{fig:causal_graph_BP}
        \vspace{-15pt}
\end{figure}

\section{Sequential Persuasion with unobserved confounding} \label{sec:spp}
In our framework of persuasion under unobserved confounding, we extend the standard BP to a sequential setting where the sender communicates with the receiver and the receiver interacts with the environment over several rounds. The index of the round is indicated by $i$ and ranges from $i=0,\ldots,T$, where $T\in\mathbb{N}$ corresponds to the final round of interaction.
The key distinction is the presence of an unobserved confounder that affects how the receiver updates their prior belief after communicating with the sender.

The state of the environment at each round $i$ is denoted by the random variable $S_i$, with realization $s_i \in \mathcal{S}$.
Each state $S_i$ is sampled independently from the same distribution $\mu$. Unlike standard BP, the true distribution $\mu$ is known only to the sender, not to the receiver.
At each $i$, the sender can observe $s_i$ and select a signaling policy $\pi_i$ from a finite set of policies denoted by $\mathcal{P}$ at each $i$.
Then, the sender commits to a signaling policy $\pi_i$, samples signal $q_i \sim \pi_i(\cdot|s_i)$, and shares them with the receiver.
The receiver holds their own belief $B_i$ about state $S_i$, which is unknown to the sender.
We consider that the receiver's belief space is finite. Let this finite space of beliefs be denoted by $\mathscr{B}$. 
The receiver selects a self-interested action:
\begin{align}
    a^*_i= \arg\max_{a \in  \mathcal{A}} \sum_{s \in \mathcal{S}} b_i(s) \rho^r(s,a),
    \label{self_tnterested_action}
\end{align}
where $b_i(s)$ is the belief $b_i$ based probability of a state realization $s$.
%\begin{remark}
%In the context of human-AI recommendation systems, the AI communicates with a human who serves as the receiver. Typically, humans interpret and process information from these recommendations or signals in categorical terms rather than as nuanced probability distributions. For instance, an individual’s belief about the likelihood of rain on a given day may range from low to medium to high, reflecting their level of confidence. Additionally, due to computational constraints, a receiver may simplify their belief space by utilizing a finite set of categories, thus reducing the complexity involved in updating their beliefs to make self-interested decisions.
%\end{remark}
After the receiver selects an action, both the sender and the receiver receive rewards $P^s_i = \rho^s(S_i,A_i)$ and $P^r_i = \rho^r(S_i,A_i)$, respectively, which are bounded by $[m^s,M^s]$ and $[m^r,M^r]$. 
Let $\rho^s_i$ and $\rho^r_i$ denote reward realizations for $P^s_i$ and $P^r_i$, respectively.
Since $\mathcal{S}$ and $\mathcal{A}$ are finite sets, we consider that $\rho^s_i$ and $\rho^r_i$ also belong to finite sets $\mathcal{R}^s$ and $\mathcal{R}^r$, respectively.

A key element in our framework is the introduction of an unobserved confounder $Z$. This variable affects how the receiver updates their belief but is observable only to the receiver.
At each $i$, any realization $z$ belongs to a finite set $\mathcal{Z}$. In the beginning, it is sampled from a distribution $z_0 \sim \eta$, and it remains constant $Z_{i+1}=Z_i$.

\begin{remark}
The variable $Z$ represents the nature of the receiver. For instance, consider the conservativeness of a planner within a decision-making framework, which exists on a spectrum from risk-averse (pessimistic) to risk-neutral (optimistic).
%Similarly, the concept of \textit{trust} in human-AI interactions varies from complete trust to total distrust. The inherent nature of the planner significantly influences their beliefs regarding the state of the system and their decision-making processes.
\end{remark}

%\begin{remark}
%The variable $Z$ represents the nature of the receiver. For instance, consider the conservativeness of a planner within a decision-making framework, which exists on a spectrum from risk-averse (pessimistic) to risk-neutral (optimistic). Similarly, the concept of \textit{trust} in human-AI interactions varies from complete trust to total distrust. The inherent nature of the planner significantly influences their beliefs regarding the state of the system and their decision-making processes.
%\end{remark}

When the receiver obtains information about the signal and the committed policy, the dynamics fo their belief update is given by the conditional distribution $   p(b_i|b_{i-1},\rho^r_{i-1},a_{i-1},q_i,z_i,\pi_i)$,
We note that the belief-update dynamics is unknown to the sender.
Considering that the receiver has no information at the initial round, we assume that they start with a uniform prior; that is, $b_{0}$ is the \textit{uniform distribution} on $\mathcal{S}$.
After the belief update, the receiver selects action $a^*_i$ according to \eqref{self_tnterested_action}.
Since $Z_i$ is unobservable by the sender, and influences the receiver's belief update, in the language of causal inference, $Z_i$ acts as an \textit{unobserved confounder}. An unobserved confounder presents a challenge to the sender in keeping track of or predicting belief updates over the rounds.
At each $i$, we denote all information available to the sender using the information vector, denoted by $\Delta_i$.
Any realization of such information vector is given by a tuple $\delta_i= (s_{0:i-1},\pi_{0:i-1},q_{0:i-1},a_{0:i-1},\rho^r_{0:i-1})$. At each $i$, the set of possible realizations of $\Delta_i$ is denoted by $\mathcal{D}_i$. 

In this framework, we consider that the sender designs a meta-policy $\boldsymbol{\pi}:= (p_0,p_{2},\ldots,p_i)$, where each $p_i : \mathcal{D}_i \times \mathcal{S} \to \mathcal{P}$ selects the signaling policy as $\pi_i=p_i(\delta_i,s_i)$.
For the ease of interpretation, we consider deterministic meta-policies. However, the analysis in this work is applicable even in the stochastic case.
The set of meta-policies is denoted by $\boldsymbol{\Pi}$.

Next, we define the performance of any meta-policy $\boldsymbol{\pi}$ as $J(\boldsymbol{\pi}):= \mathbb{E}^{\boldsymbol{\pi}}[\sum_{i=0}^T \rho^s(S_i,A_i)]$,
%\begin{align}
%    J(\boldsymbol{\pi}):=\;& \mathbb{E}^{\boldsymbol{\pi}}[\sum_{i=0}^T \rho^s(S_i,A_i)],
%\end{align}
where $\mathbb{E}^{\boldsymbol{\pi}}(\cdot)$ is the expectation over the joint distribution induced by the meta-policy $\boldsymbol{\pi}$, the belief update dynamics and the self-interested receiver action selected according to \eqref{self_tnterested_action}.
\vspace{5pt}

%\begin{problem} \label{prb:meta-policy}
%The goal of the sender is to compute the best signaling meta-policy %$\boldsymbol{\pi}^*$, i.e.,
%\begin{align}
%\boldsymbol{\pi}^* = \arg\max_{\boldsymbol{\pi} \in \boldsymbol{\Pi}} J(\boldsymbol{\pi}).
%\end{align}    
%\end{problem}

\begin{problem} \label{prb:meta-policy}
The goal of the sender is to compute the best signaling meta-policy $\boldsymbol{\pi}^*$, defined as $\boldsymbol{\pi}^* = \arg\max_{\boldsymbol{\pi} \in \boldsymbol{\Pi}} J(\boldsymbol{\pi})$.
\end{problem}

\section{Construction of POMDP Structure} \label{sec:pomdp}

This section demonstrates how the sequential persuasion process (SPP) with unobserved confounding can be formulated as a POMDP.
The sender serves as an agent who makes a decision, the receiver becomes part of an environment, and the sender's incomplete knowledge of the receiver's belief and the unobserved confounder are considered unobserved quantities.
Then, we use this framework to formulate a problem equivalent to Problem \ref{prb:meta-policy}.

%We then consider a class of observation-based control strategies for the planner and formulate an optimization problem to find the best such strategy.
%We show that solving this problem is equivalent to solving Problem \ref{prb:meta-policy}.

We reformulate the SPP as a system that evolves over discrete time steps, denoted by $t=0, \dots, T+1$, where $T$ is the final round of interaction in the persuasion process.
Furthermore, the definition of the state of this system will give intuition behind a horizon of $T+1$ time steps.
In our analysis, we show how the newly formulated system is a valid POMDP.
%To show that this system is a POMDP, we first construct the state and action at each $t$ based on the SPP and then prove that its evolution is Markovian.
We use upper-case letters to denote random variables and lower-case letters to denote their corresponding realizations.

At each $t$, let $u_t$ denote the control action in this system.
Starting with $i=t=0$, the state of the system $x_0$ is realized and the observation $y_0$ is available.
The agent then implements a control action $u_0=\pi_0$, which is the signaling policy.
This concludes the first time step, after which the equivalent system evolves to the next time step.
Although the equivalent system evolves, the round $i=0$ in the SPP is not complete.
At the end of the time step $t=0$, only the state $s_0$ and the sender's signal $\pi_0$ have been realized in the SPP.
It is only after the equivalent system evolves to $t=1$ that the receiver in the SPP updates their belief in round $i=0$ and then selects action $a_0$ in the SPP.
This is illustrated in Fig. \ref{fig:reformulation}.
\begin{figure}[t]
    \centering
    \includegraphics[width=\linewidth]{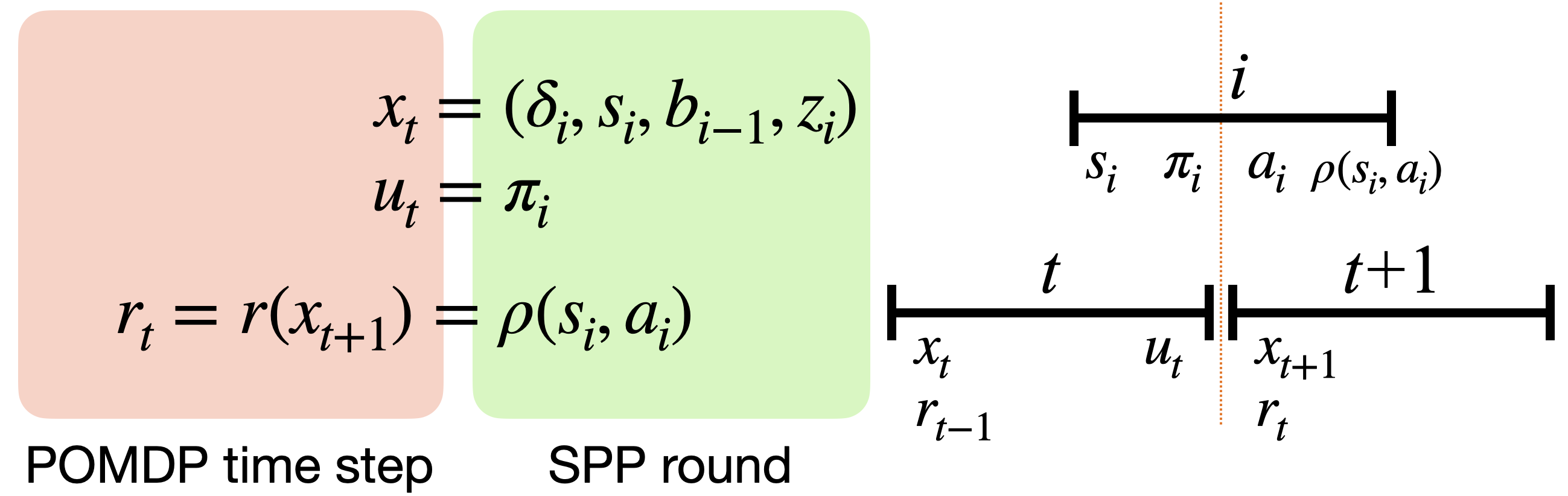}
    \caption{System reformulation: SPP to POMDP.}
    \label{fig:reformulation}
    \vspace{-10pt}
\end{figure}

We define the initial state of the system as $x_0=(b_{-1},s_0,z_0)$, with observation $y_0=(s_0)$.
At each time step $t=1, \dots, T$, the state consists of the tuple $x_t=(\delta_i,s_i,b_{i-1},z_i)$ with the observation $y_t=(\delta_i,s_i)$.
At the final time $t=T+1$, the state is given by $x_{t}=(\delta_{i+1},b_i)$, with observation $y_{t}=(\delta_{i+1})$.
At each $t$, let $\mathcal{X}_t$ be the set of all state realizations.
Note that the number of elements of information vector $\delta_{i}$, which is a component of the state realization $x_t$ and the observation $y_t$, grows over time.
Consequently, for any distinct time steps $t$ and $t'$, the sets $\mathcal{X}_t$ and $\mathcal{X}_{t'}$, as well as $\mathcal{Y}_t$ and $\mathcal{Y}_{t'}$ are disjoint.
We define the state space and the observation space of the system as the disjoint unions by $\mathcal{X}= \cup_{t=0}^{T+1} \mathcal{X}_t$ and $\mathcal{Y}= \cup_{t=0}^{T+1} \mathcal{Y}_t$, respectively. 

\begin{lemma} \label{lem:pomdp}
    The evolution of the state of this system is Markovian. Hence, the system is a POMDP.
\end{lemma}
\begin{proof}
    The state at time $t+1$ is given by $x_{t+1}=(\delta_{i+1},s_{i+1},b_i,z_{i+1})$. Since $s_{i+1} \sim \mu$, it evolves independently and $z_{i+1}=z_i$ do not depend on the action at time $u_t=\pi_i$. 
    The belief update, given by  $p(b_i|b_{i-1},\rho^r_{i-1},a_{i-1},q_i,z_i,\pi_i)$ depends on $(b_{i-1},\rho^r_{i-1},a_{i-1},q_i,z_i)$, which are all components of the state $x_{t}$ and the action $u_t=\pi_i$. 
    Consider the information vector $\delta_{i+1}=(\delta_{i},s_i,q_i,a_i,\rho^r_i,\pi_i)$.
    We know that $s_i\sim \mu$ is drawn independently.
    The receiver's action $a_i$ is a function of the belief $b_i$, hence it is purely state and action-dependent.  
    As the reward for the receiver $\rho^s_i$ is based on $a_i$ and $s_i$, it is also purely state and action-dependent.
    Given $s_i$, the action $\pi_i$ completely determines the signal as $q_i \sim \pi_i(\cdot|s_i)$.  
    This shows that the state $x_t$ and the action $u_t$ are enough to predict a distribution on the next state $x_{t+1}$. Hence the state evolution is Markovian, that is, $p(x_{t+1}|x_{0:i},u_{0:i})= p(x_{t+1}|x_t,u_t)$. Since the agent cannot observe the full state $x_t$, this system is a POMDP.
\end{proof}

\begin{remark}
    Since state space $\mathcal{X}$ is a disjoint union of sets, the Markov transitions in the POMDP are such that, at each $t$, $p(f_{t+1}|f_t,u_t)=0$ if $f_{t+1} \notin \mathcal{X}_{t+1}$.
\end{remark}

% \begin{proof}
% At each $i$, we consider the component-wise update and show that the evolution of the state to $x_{t+1}$ depends only on the current state $x_{t}$ and the action $u_t=\pi_i$. At each $i=1,\ldots,T-1$, the expressions for the component wise update are given by 
% \begin{align}
% \nonumber \delta_{i}&\;=(s_{0:i},\pi_{0:i},q_{0:i-1},a_{0:i},\rho^r_{0:i-1})\\
% &\;=(\delta_{i-1},s_i,\pi_i,q_{i-1},a_{i-1},\rho^r_{i-1}),\\
% q_{i}&\; \sim \pi_i(\cdot|S_i),\label{signal_sampling}\\
% b_i&\;=h(b_{i-1},\rho^r_{i-1},q_i,z_i,\pi_i),\\
% a_{i}&\;= \arg\max_{a \in  \mathcal{A}} \sum_{s \in \mathcal{S}} b_{i}(s) \rho^r(s,a),\\
% \rho^r_i&\;=\rho^r(s_{i},a_{i}),\\
% s_{i+1}&\; \sim \mu(\cdot),\\
% z_{i+1}&\;=z_i,
% \end{align}
% where, at $i=T$, the evolutions are the same except that there is no evolution to $X_{i+1}$ and $Z_{i+1}$. 
% We can clearly see from the component-wise state update expressions that the state evolution is Markovian:
% \begin{align}
%     p(X_{t+1}|X_{0:i},u_{0:i})= p(X_{t+1}|X_t,u_t).
% \end{align}
% \HB{Rather than showing all the equation and saying it's obvious, it would be better to explicitly point out what are affecting the state evolution. I will take a look again later.}
% \end{proof}

%%
At each $t=0, \ldots, T$, the agent selects the action $u_t$ using a control law $g_t: \
\mathcal{Y} \to \mathcal{P}$ as $u_t = g_t(y_t)$.
At each $t$, any control law such that for any action $u\in  \mathcal{U}$, if $y_t \notin \mathcal{Y}_t$, then $p(u|y_t)=0$, is a valid control law.
The feasible set of control laws at time $t$ is denoted by $\mathcal{G}_t$.
The tuple of control laws denotes the control strategy of the planner $\boldsymbol{g} := (g_0,\dots,g_{n-1})$, where $\boldsymbol{g} \in \mathcal{G}$ and $\mathcal{G} = \prod_{t = 0}^{n-1}\mathcal{G}_{t}$.

\begin{lemma} \label{lem:strategy}
The set of all observation-based control strategies $\mathcal{G}$ in the constructed POMDP is equivalent to $\boldsymbol{\Pi}$ the set of all meta-policies in the persuasion framework.    
\end{lemma} 

\begin{proof}
We recall that in the persuasion framework, the sender’s meta-policy $\boldsymbol{\pi}$ defines a tuple of functional mappings $p_i: \mathcal{D}_i \times \mathcal{S} \to \mathcal{P}$.
At each $i$, the signaling policy $\pi_i$ is set by evaluation of $p_i$ at $\Delta_i=\delta_i$ as $\pi_i=p_i(s_{0:i-1},\pi_{0:i-1},q_{0:i-1},a_{0:i-1},\rho^r_{0:i-1},q_i)$.
%\begin{align}
%    \pi_i=p_i(s_{0:i-1},\pi_{0:i-1},q_{0:i-1},a_{0:i-1}\rho^r_{0:i-1},q_i).\label{all_p_i}
%\end{align}
In the case of observation-based control strategies in the POMDP setting, the control law $g_t$ selects the action $u_t=\pi_i$ based on $u_t\;=g_t(y_t)=g_t(s_{0:i-1},\pi_{0:i-1},q_{0:i-1},a_{0:i-1},\rho^r_{0:i-1},q_i)$.
%\begin{align}
    %u_t\;=g_t(y_t)=g_t(s_{0:i-1},\pi_{0:i-1},q_{0:i-1},a_{0:i-1},\rho^r_{0:i-1},q_i).\label{all_g_i}
%\end{align}
With this, we observe that at each $t=i$, the selected signaling policy is identical in both frameworks. 
Consequently, the space of all sender meta-policies $\boldsymbol{\Pi}$ coincides with the set of all observation-based control strategies $\mathcal{G}$ for the planner in the POMDP framework.

\end{proof}
%\begin{remark} \label{remark_POMDP_reward}
    %Since the set of all meta-policies $\boldsymbol{\Pi}$ coincides with the set of all observation-based control strategies $\mathcal{G}$, any strategy $\boldsymbol{g} \in \mathcal{G}$ is a valid meta-policy. 
    %From now on, we will use the notion of strategy in our analysis instead of meta-policy.
    %Furthermore, we will use $u_t$, the agent's action, to indicate the sender's signaling policy. 
%\end{remark}

\begin{remark} \label{remark_POMDP_reward}
    Since the set of all meta-policies $\boldsymbol{\Pi}$ coincides with the set of all observation-based control strategies $\mathcal{G}$, any strategy $\boldsymbol{g} \in \mathcal{G}$ is a valid meta-policy. 
    %From now on, we will use the notion of strategy in our analysis instead of meta-policy.
    %Furthermore, we will use $u_t$, the agent's action, to indicate the sender's signaling policy. 
\end{remark}

After the agent selects an action and the system evolves to time $t+1$, the agent receives a reward denoted by $r^s_t$. 
To align the system description with a standard POMDP formulation, we construct the reward function to depend on the future state of the system.
The reward for the agent at time $t$ is a mapping $r_t: \mathcal{X}_{t+1} \to [m^s,M^s]$ that satisfies $r_t(x_{t+1})=\rho^s(s_i,a_i)$, for $i=t$.
Since the reward $\rho^S$ for the persuasion process belongs to a finite set $\mathcal{R}^s$, the reward $r^s_t$ also belongs to the set $\mathcal{R}^s$. 
Since the system evolves only until time $T+1$, we set the reward $r^s_{T+1}$ identically to zero.

In the POMDP setting, we define the performance of any strategy $\boldsymbol{g}$ as $J(\boldsymbol{g})= \mathbb{E}^{\boldsymbol{g}} \left[\sum_{t=0}^T r_t(X_{t+1})\right]$,
%\begin{align}
%    J(\boldsymbol{g})= \mathbb{E}^{\boldsymbol{g}} \left[\sum_{t=0}^T r_t(X_{t+1})\right],
%\end{align}
where $\mathbb{E}^{\boldsymbol{g}}$ denotes the expectation on all the random variables with respect to the probability distributions generated by the choice of control strategy $\boldsymbol{g}$. 

\begin{problem} \label{prb:pomdp}
The goal of the agent is to compute the best observation-based control strategy $\boldsymbol{g}^*$, defined as $\boldsymbol{g}^* = \arg\max_{\boldsymbol{g} \in \boldsymbol{G}} J(\boldsymbol{g})$.
\end{problem}

%\begin{problem} \label{prb:pomdp}
%The goal of the agent is to compute the best observation-based control strategy $\boldsymbol{g}^*$, i.e.,
%\begin{align}
%    \boldsymbol{g}^* = \arg\max_{\boldsymbol{g} \in \boldsymbol{G}} J(\boldsymbol{g}).
%\end{align}
%\end{problem}

\begin{theorem}
There exists a POMDP that is equivalent to the persuasion framework, and the optimal observation-based control strategy for this POMDP corresponds to a solution of Problem \ref{prb:meta-policy}.    
\end{theorem}

\begin{proof}
We reformulate the sequential persuasion process as a system evolving over discrete time steps. 
We define the notions of state, observation, and action that are consistent with a general POMDP framework. 
Furthermore, in Lemma \ref{lem:pomdp}, we show that the constructed system is indeed a valid POMDP.
In Lemma \ref{lem:strategy}, we establish that the set of observation-based control strategies $\mathcal{G}$ coincides with $\boldsymbol{\Pi}$, the set of meta-policies considered.
Next, we construct a reward function for the planner in the POMDP framework to match that of the sender in the persuasion framework.
Thus, the solution to Problem \ref{prb:pomdp} is a valid solution to Problem \ref{prb:meta-policy} and Problem \ref{prb:pomdp} is equivalent to Problem \ref{prb:meta-policy}.       
\end{proof}

Note that to formulate and solve Problem \ref{lem:pomdp}, we need to be able to compute the performance of any observation-based control strategy.
At each $t$, let $\tau_t= (x_{0:t},y_{0:t},u_{0:t})$ be the trajectory of the system. 
We denote the space of all trajectories at time $t$ by $\mathcal{T}_t$.
The observable trajectory of the system is denoted by $\tau^{o}_t= (y_{0:t},u_{0:t})$.
The space of all observable trajectories at time $t$ is $\mathcal{T}^o_t$. 
Any strategy $\boldsymbol{g}$ induces a distribution $p^{\boldsymbol{g}}_t(\cdot)$ over $\mathcal{T}_t$.
Similarly, it induces a distribution $p^{\boldsymbol{g},o}_t(\cdot)$ over $\mathcal{T}^o_t$.
In our analysis to compute the performance of any strategy,  we make use of the conditional independencies among variables in the POMDP based on their causal relationship. 
The causal relationships in the formulated POMDP are given by the causal graph illustrated in Fig. \ref{fig:POMDP_Causal_graph}. 

%In Section \ref{sec:OPE}, we show that the performance of a strategy depends on the probability measure induced by this strategy on the trajectory space of the POMDP.
%In our analysis, we make use of the conditional independencies among variables in the POMDP based on their causal relationship. 
%The causal relationships in the formulated POMDP are given by the causal graph illustrated in Fig. \ref{fig:POMDP_Causal_graph}. 

\begin{figure}[t]
    \centering    \includegraphics[width=0.8\linewidth]{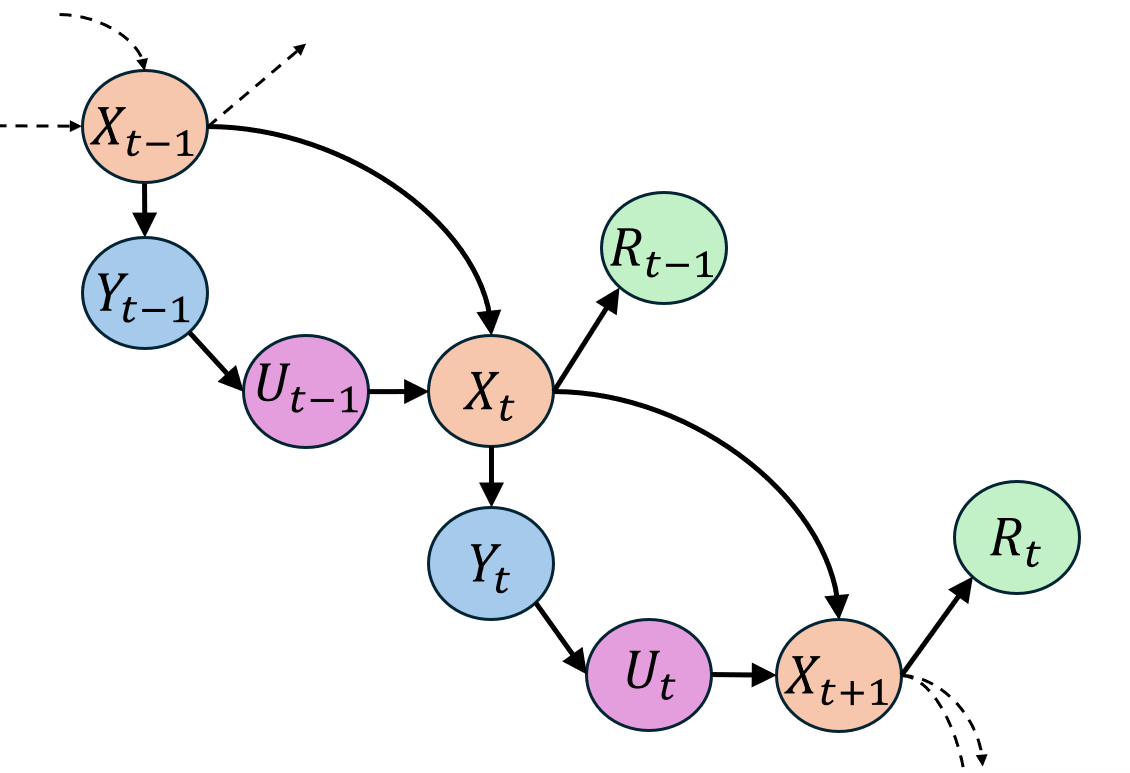}
    \caption{Causal graph of POMDP.}
    \label{fig:POMDP_Causal_graph}
        \vspace{-15pt}
\end{figure}

In this POMDP, the unobserved state $x_t$ plays the role of an \textit{unobserved confounder} between the action $u_t$ and the next state $x_{t+1}$.
Since the exact dynamics of the POMDP are unknown and an unobserved confounder exists, we do not have access to $p^{\boldsymbol{g}}_i(\cdot)$ or $p^{\boldsymbol{g},o}_t(\cdot)$.
Consequently, we can only hope to compute $J(\boldsymbol{g})$ based on a logged dataset.
To this end, we consider that we have access to a dataset $\boldsymbol{D}^b=\{\tau^{o,k}_{T+1}\}_{k=0}^n$, which is a collection of \textit{n}-observable trajectories, generated by an unknown behavioral control strategy $\boldsymbol{g}^b$.
Typically, OPE techniques are used to compute the performance of any given strategy based on $\boldsymbol{D}^b$.
However, the unobserved confounder, by its nature, poses a challenge to OPE. In the next section, we analyze how unobserved confounding in this context is handled.

%Furthermore, we demonstrate how the performance of any given evaluation strategy $\boldsymbol{g}^e$ be computed based on $\mathcal{D}^b$ and how the issue of unobserved confounding is adjusted by using the proximal learning technique.

\section{Off-Policy Evaluation}\label{sec:OPE}

In this section, we show how to compute the performance of any evaluation strategy $\boldsymbol{g}^e \in \mathcal{G}$ to facilitate the formulation of Problem \ref{lem:pomdp}.
Firstly, we introduce the notation that we use in our analysis.

\subsection{Notations}
In this subsection, we introduce an example of vector and matrix notation for probability distributions, which we use in our analysis.
At each $t$, consider any state realization $x_t \in \mathcal{X}$ and next state realization $x_{t+1} \in \mathcal{X}$. %such that $t,j \in \{1.\ldots,|\mathcal{F}|\}$.
For some action $u_t$, the transition probability between any two state realizations $x_t$ and $x_{t+1}$, is denoted by $p(x_{t+1}|x_t,u_t)$.
In our notation, $P(X_{t+1}|x_t,u_t)$ is a column vector of dimension $|\mathcal{X}|\times 1$ given by $P(X_{t+1}|x_t,u_t)= (p(x^1_{t+1}|x_t,u_t),\ldots,p(x^{|\mathcal{X}|}_{t+1}|x_t,u_t))$, where $x^k_{t+1}$ is $k$-th possible realization of $X_{t+1}$. 
Similarly, $P(x_{t+1}|X_t,u_t)$ is a row vector of dimension $1 \times|\mathcal{X}|$ given by $P(x_{t+1}|X_t,a_t)= (p(x_{t+1}|x^1_t,u_t),\ldots,p(x_{t+1}|x^{|\mathcal{X}|}_t,u_t))$. 
In our analysis, the multiplication of a pair of vectors of appropriate dimension is their \textit{scalar product}. 
The notation $P(X_{t+1}|X_t,U_t)$ indicates a $|\mathcal{X}| \times |\mathcal{X}|$ matrix.
% The element in the $k^{th}$ row and $t^{th}$ column of this matrix is given by $p(x^k_{t+1}|x^j_t,u_t)$.
When terms in the conditioning are not of the same dimension, this would be the appropriate rectangular matrix.
In our analysis, the multiplication of two matrices refers to their \textit{algebraic matrix multiplication}.

Our analysis also makes use of the conditional independencies among variables in the POMDP, based on their causal relationship illustrated in Fig. \ref{fig:POMDP_Causal_graph}. 
As an example, the causal relationship corresponding to the Markovian evolution of the state given by $p(x_{t+1}|x_{0:i},u_{0:i})= p(x_{t+1}|x_t,u_t)$, is denoted by $x_{t+1} \perp\!\!\!\!\perp (x_{0:i-1},u_{0:i-1}) | (x_t,u_t)$.

\subsection{Evaluation using proximal learning}

We begin by highlighting the dependence of the performance of any strategy on the probability measure it induces on the trajectory space of the POMDP.
We expand the expression for performance $J(\boldsymbol{g}^e)$ which is given by
\begin{align}
    J(\boldsymbol{g}^e)= \mathbb{E}^{\boldsymbol{g}^e} \left[\sum_{t=0}^T r_t
    (X_{t+1})\right]=\sum_{t=0}^{T}\sum_{r^s_t\in \mathcal{R}^s} r^s_t \cdot p^{\boldsymbol{e}}(r^s_t),\label{r_pover_r_conversion}
\end{align}
where $p^{\boldsymbol{e}}(r^s_t)$ is the probability of a reward realization $r^s_t$ at time $t$ under the strategy $\boldsymbol{g}^e$.
We consider $p^e(r^s_t)$ and use the law of total probability and Bayes' theorem to get
\vspace{-2pt}
\begin{align}
p^{\boldsymbol{e}}(r^s_t)=\sum_{\tau_{t+1}}\;p^{\boldsymbol{e}}(r^s_t|x_{t+1})\;p^{\boldsymbol{e}}(\tau_{t+1}),\label{p_e_r_ierm2}
\end{align}
\vspace{-2pt}
%\begin{align}
%p^{\boldsymbol{e}}(r^s_t)&=\sum_{\tau_{t+1}}\;p^{\boldsymbol{e}}(r^s_t|\tau_{t+1})\; p^{\boldsymbol{e}}(\tau_{t+1}),\\
%&=\sum_{\tau_{t+1}}\;p^{\boldsymbol{e}}(r^s_t|x_{t+1})\;p^{\boldsymbol{e}}(\tau_{t+1}),\label{p_e_r_ierm2}
%\end{align}
where $\sum_{\tau_{t+1}}$ indicates summation over all possible trajectories $\tau_{t+1} \in \mathcal{T}_{t+1}$ and we have use the conditional independence $r^s_{t} \perp\!\!\!\!\perp \tau_{t} | x_{t+1}$.
%The expression in \eqref{p_e_r_ierm2} clearly shows the dependence of $J(\boldsymbol{g}^e)$ on $p^{\boldsymbol{e}}(\tau)$.
Furthermore, we note that knowledge of $p^{\boldsymbol{e}}(\tau_{t+1})$, will allow us to compute $p^{\boldsymbol{e}}(r^s_t)$ at each $t$, and, in turn, the performance.

In the following analysis, we demonstrate the proximal learning method for sequential decision-making problems presented in \cite{tennenholtz2020off} for the reformulated POMDP.
%The essence of proximal learning is to find \textit{proxies} to infer the average influence of the unobserved confounder on the performance of a strategy.
%Firstly, we present a result that provides a method for computing the distribution $p^{\boldsymbol{e}}(\cdot)$ on the reward $r$ at each time $t$ as a function of just the observable trajectory $\tau^o_t$.
%In the proof, we show how past and current observations of the state can be used as proxies at each $t$, to infer the performance of an evaluation strategy.
For the exposition, we consider the following assumptions. 

\begin{assumption}\label{assm_extraobs}
We assume that we have access to an observation before the initial time $t=0$, denoted by $y_{-1}$, which can be used as one of the proxies. 
\end{assumption}
Recall that in the persuasion process, at the beginning, the sender considers a uniform prior to the receiver's belief. This uniform prior can be considered as observation $y_{-1}$. 

\begin{assumption}\label{assm_tnvrtble}
At each $t=1,\ldots,T$, the matrices $P^{\boldsymbol{b}}(Y_t \mid Y_{t-1},u_t)$ and $P^{\boldsymbol{b}}(X_t \mid Y_{t-1}, u_t)$ are invertible.    
\end{assumption}

This assumption ensures that the observation at different time steps carries over enough information about the state.

At each time $t=1, \ldots,T$, for any observable trajectory $\tau^o_t \in \mathcal{T}^0_t$, we define the weight matrix $W_t(\tau^o_t)$ given by
\begin{align}
    W_t(\tau^o_t)= P^{\boldsymbol{b}}(Y_t|Y_{t-1},u_t)^{-1}\cdot P^{\boldsymbol{b}}(Y_t,y_{t-1}|Y_{t-2},u_{t-1}).
\end{align}
For time $t=0$, we define the weight matrix $W_0(\tau^0_0)$ for any $\tau^o_0 \in \mathcal{T}^0_0$ as $W_0(\tau^0_0)=P^{\boldsymbol{b}}(Y_0|u_0,Y_{-1})^{-1}\cdot P^{\boldsymbol{b}}(Y_0)$. 

\begin{theorem}\label{Theorem_reward_distrib}
Under Assumption \ref{assm_extraobs} and Assumption \ref{assm_tnvrtble}, at each $t$, 
the reward distribution under any evaluation strategy $\boldsymbol{g}^e$ can be computed based on $\mathcal{D}^{\boldsymbol{b}}$ and is given by
\begin{align}
    \nonumber &P^{\boldsymbol{e}}(r_t)=\\
    &\sum_{\tau^o_{t+1}} \Pi_{k=0}^{t+1}\;p^{\boldsymbol{e}}(u_k|y_k)\; P^{\boldsymbol{b}}(r_t,y_{t+1}|u_{t+1},Y_{t})\;
    \Pi_{k=0}^{t+1} W(\tau^o_k).
\end{align}
\end{theorem}

\begin{proof}
\vspace{-10pt}
The proof proceeds in two steps.
We expand the expression for the reward distribution and decompose it into two components based on dependence on the strategy.
The first component includes strategy-dependent distributions on observed variables, which can be computed given a specific evaluation strategy $\boldsymbol{g}^e$.
On the other hand, strategy-independent distributions can be over a combination of observable and unobservable variables. 
Next, we show how to compute the strategy-independent component from observational data using proximal learning.
Details of the mathematical arguments are provided in Appendix A.
\end{proof}

\section{Application Example}\label{Application}

In this section, we elaborate on a framework based on SPP to elucidate how AI can communicate effectively within the context of industrial safety management. This framework aims to influence the behavior of warehouse associates engaged in the transportation of packages through various means, including forklifts, conveyor systems, and manual handling. Packages may carry specific handling instructions, such as ``fragile," ``temperature-sensitive," or ``do not turn upside down."
The primary objective for associates is to maximize operational efficiency to enhance their incentives, while safety managers must ensure both package throughput and the safety of workers. The daily demand and the inherent characteristics of the packages define the system’s current state. 
%Equipped with advanced computational resources, safety managers can accurately estimate the distribution from which this state is derived, thereby possessing comprehensive knowledge of the state distribution—information not directly accessible to the associates.
To facilitate safe practices, safety managers communicate with associates via personal devices, dispatching signals that range from ``extremely cautious" to ``at ease." These signals are critical as they guide associates in updating their beliefs before decision-making. For instance, on high-demand days, cautionary signals are prevalent, whereas ``at ease” messages are more common during periods of lower demand. It is essential to recognize that individual risk tolerance influences how associates interpret these signals; a risk-averse associate may equate a cautious signal with a high-demand scenario, while a risk-seeking associate might undervalue caution even in similar contexts. This variability underscores the importance of considering past interactions and rewards, which can significantly shape the interpretation and effectiveness of safety communication strategies.

\section{Conclusion}
\label{sec:conclusion}
\begin{comment}    
In this paper, we developed a framework to study a SPP with unobserved confounding affecting the belief updates of the receiver. 
We formulated an equivalent POMDP, whose optimal observation-based control strategy is the optimal signaling strategy for this persuasion framework.
Later, we focus on the off-policy evaluation of any signaling strategy designed by the sender to select the signaling policy at each round.
We show how the notion of proxy variables, when used for the equivalent POMDP to adjust for unobserved confounding in OPE, allows us to evaluate the performance of any feasible strategy for the considered persuasion framework.
To illustrate the framework, we discuss the example of interaction between a traffic manager and CAVs.
Some limitations of this work to be addressed in future work include improving computational tractability, incorporating various  OPE techniques like sensitivity analysis \cite{kallus2020confounding}, bridge functions\cite{bennett2024proximal}, extending it
to the continuous domain of policies, optimizing to get best meta-policy\cite{kausik2024offline} and applying it to practical applications involving complex human tasks.
\end{comment}

In this paper, we developed a framework to study an SPP that incorporates unobserved confounding affecting the belief updates of the receiver. We formulated an equivalent POMDP model, where the optimal observation-based control strategy corresponds to the optimal meta-policy for the SPP.
Furthermore, we demonstrate off-policy evaluation using proximal learning to showcase policy evaluation under unobserved confounding. We provide an industrial safety- management example to illustrate practical applications.

Future works could consider enhancing computational tractability for implementation, integrating different OPE techniques (sensitivity analysis \cite{kallus2020confounding}, bridge functions \cite{bennett2024proximal}), and extend this framework to continuous spaces to be applicable to complex human decision-making tasks.
 
\section*{Acknowledgments}
The authors would like to thank Dr. Aditya Dave for valuable discussions on the formulation and solution approach.

\bibliographystyle{ieeetr}
\bibliography{References, Latest_IDS}
%\input{main.bbl} 
% This incluedes the latest ids bib file

\section*{Appendix A - Proof of Theorem 2}

The essence of proximal learning is to find \textit{proxies} to infer the average influence of the unobserved confounder on the performance of a strategy.
Firstly, we present a result that provides a method for computing the distribution $p^{\boldsymbol{e}}(\cdot)$ on the reward $r$ at each time $t$ as a function of just the observable trajectory $\tau^o_t$.
We present a detailed proof of Theorem \ref{Theorem_reward_distrib} and show how past and current observations of the state can be used as proxies at each $t$, to infer the performance of an evaluation strategy.

We begin by expanding the strategy dependent distribution $p^{\boldsymbol{e}}(\tau_{t+1})$ for any $\tau_t=(u_{0:t},y_{0:t},x_{0:t})$ in \eqref{p_e_r_ierm2} as
\begin{align}
    &\nonumber p^{\boldsymbol{e}}(\tau_{t+1})\\
    =&\;p^{\boldsymbol{e}}(u_{t+1}|y_{t+1},x_{t+1},\tau_t)\;p^{\boldsymbol{e}}(y_{t+1},x_{t+1},\tau_t),\\
    =&\;p^{\boldsymbol{e}}(u_{t+1}|y_{t+1})\;p^{\boldsymbol{e}}(y_{t+1}|x_{t+1},\tau_t)\;p^{\boldsymbol{e}}(x_{t+1},\tau_t),\label{p^e_tau_action}\\
    =&\;p^{\boldsymbol{e}}(u_{t+1}|y_{t+1})\;p^{\boldsymbol{e}}(y_{t+1}|x_{t+1})\;p^{\boldsymbol{e}}(x_{t+1}|\tau_t)\;p^{\boldsymbol{e}}(\tau_t),\label{p^e_tau_obs}\\
    =&\;p^{\boldsymbol{e}}(u_{t+1}|y_{t+1})\;p^{\boldsymbol{e}}(y_{t+1}|x_{t+1})\;p^{\boldsymbol{e}}(x_{t+1}|x_t,u_t)\;p^{\boldsymbol{e}}(\tau_t),\label{p^e_tau_state}\\
    \nonumber=&\;\Pi_{k=0}^{t+1}\;p^{\boldsymbol{e}}(u_k|y_k)\;
    \Pi_{k=0}^{t+1}\; p^{\boldsymbol{e}}(y_k|x_k)\;\\
    &\hspace{60pt} \cdot\Pi_{k=0}^{t}\;p^{\boldsymbol{e}}(x_{k+1}|x_k,u_k)\;p^{\boldsymbol{e}}(x_{0}),\label{p_e_r_ierm2_exp} 
\end{align}
where we use the fact that the control strategy is observation-based in \eqref{p^e_tau_action}.  The transition to \eqref{p^e_tau_obs}, is based on how the observation $y_{t+1} \perp\!\!\!\!\perp \tau_t |x_{t+1}$. Furthermore, the Markovian evolution of the state is used to obtain \eqref{p^e_tau_state}.

We substitute the expanded joint probability in \eqref{p_e_r_ierm2} to obtain
\begin{align}
&\nonumber p^{\boldsymbol{e}}(r^s_t)\\
\nonumber &\;= \sum_{\tau_{t+1}}\; p^{\boldsymbol{e}}(r^s_t|x_{t+1})\;\Pi_{k=0}^{t+1}\;p^{\boldsymbol{e}}(u_k|y_k)\;\Pi_{k=0}^{t+1}\;p^{\boldsymbol{e}}(y_k|x_k)\;\\
&\hspace{60pt}\cdot\Pi_{k=0}^{t}\;p^{\boldsymbol{e}}(x_{k+1}|x_k,u_k)\;p^{\boldsymbol{e}}(x_{0}),\\
\nonumber&\;=\sum_{\tau_{t+1}}\;\Pi_{k=0}^{t+1}\;p^{\boldsymbol{e}}(u_k|y_k)\;\;p^{\boldsymbol{b}}(r^s_t|x_{t+1})\;\Pi_{k=0}^{t+1}\; p^{\boldsymbol{b}}(y_k|x_k)\;\\
&\hspace{60pt} \cdot \Pi_{k=0}^{t}\;p^{\boldsymbol{b}}(x_{k+1}|x_k,u_k)\;p^{\boldsymbol{b}}(x_{0}), \label{shift_super_b_io_e}
\end{align}
where, in \eqref{shift_super_b_io_e} we change the superscript on all the terms except 
$p^{\boldsymbol{e}}(u_t|y_t)$.
It is essential to note that this strategy-dependent term can be computed based on the observation.
The reward $r^s_t$ is dependent only on $(x_{t+1})$, hence, independent of the agent's strategy. 
Furthermore, the evolution of the state of the system is Markovian, so the transition term is also independent of the strategy.
Consequently, we can compute all terms independent of the strategy from a dataset generated by any control strategy.
Given that we have a dataset generated by the behavioral strategy $\boldsymbol{g}^b$, we change the superscripts to indicate that we compute these terms from the behavioral data $\boldsymbol{D}^b$.

At each $t$, since the observation $y_t$ is independent of the action $u_t$ given the state $x_t$ under measure $p^{\boldsymbol{b}}$, we rephrase the expression in \eqref{shift_super_b_io_e} as
\begin{align}
\nonumber&\nonumber p^e(r^s_t)\;= \sum_{\tau_{t+1}}\;\Pi_{k=0}^{t+1}\;p^{\boldsymbol{e}}(u_k|y_k)\;p^{\boldsymbol{b}}(r^s_t|x_{t+1})\;p^{\boldsymbol{b}}(y_{t+1}|x_{t+1})\;\\
&\hspace{60pt}\cdot\Pi_{k=0}^{t}\;p^{\boldsymbol{b}}(x_{k+1},y_k|x_k,u_k)\;p^{\boldsymbol{b}}(x_{0}). 
\end{align}
Furthermore, the reward $r^s_t$ is independent of the observation $y_{t+1}$ at time $t+1$, given the tuple $(x_{t+1},u_{t+1})$.
We use this condition to combine the remaining terms to derive
\begin{align}
&\nonumber p^e(r^s_t)\;=\;\sum_{\tau_{t+1}}\;\Pi_{k=0}^{t+1}\;p^{\boldsymbol{e}}(u_k|y_k)\;p^{\boldsymbol{b}}(r^s_t,y_{t+1}|x_{t+1},u_{t+1})\;\\
&\hspace{60pt} \cdot\Pi_{k=0}^{t}\;p^{\boldsymbol{b}}(x_{k+1},y_k|x_k,u_k)\;p^{\boldsymbol{b}}(x_{0}). \label{p^e(r^s_i)_scalar}
\end{align}

We use the matrix and vector notation for the unobserved state of the POMDP and retain the summation over $\tau^o_{i+1}$, which is the observable component of $\tau_{t+1}$. This gives us the following expression:
\begin{align}
\nonumber& p^e(r^s_t)\;=\sum_{\tau^o_{i+1}}\;\Pi_{k=0}^{t+1}\;p^{\boldsymbol{e}}(u_k|y_k)\;p^{\boldsymbol{b}}(r^s_t,y_{t+1}|X_{t+1},u_{t+1})\;\\
&\hspace{60pt}\cdot \Pi_{k=0}^{t}\;p^{\boldsymbol{b}}(X_{k+1},y_k|X_k,u_k)\;p^{\boldsymbol{b}}(X_{0}). \label{p^e(r^s_i)_vector}
\end{align}

%%%%%%%%%%%%%%%%%%%%%%%%%%%%%%%%%%%%%%%%%%%%%%%%

We now discuss how terms in \eqref{p^e(r^s_i)_vector}, which involve the unobserved state, can be inferred from proxies.
%At each $t$, the essence of proximal learning is to find two \textit{proxies} to the state $x_t$, which acts as the unobserved confounder in \eqref{p^e(r^s_i)_vector}.
In our subsequent analysis, we show that for each $t$, the pair $(y_t,y_{t-1})$ are valid proxies to $x_t$. 
First, we consider the following expression:
\begin{align}
    \nonumber&P^b({X}_{t+1},{y}_t |{Y}_{t-1},{U}_t)=\\
    &\hspace{30pt}P^{b}({X}_{t+1},{y}_t|{X}_t,{Y}_{t-1},{u}_t) \cdot P^{b}({X}_t|Y_{t-1},{u}_t),
\end{align}
where we incorporate $X_t$ into the right-hand side based on the law of total probability and Bayes theorem.
By Markovian evolution of state and the state-observation relation, $(x_{t+1},y_t) \perp\!\!\!\!\perp y_{t-1}| (x_t,u_t)$. Hence, this expression reduces to the following:
\begin{align}
    \nonumber&P^b({X}_{t+1},{y}_t |{Y}_{t-1},{u}_t)=\\
    &\hspace{30pt}P^{b}({X}_{t+1},{y}_t|{X}_t,{u}_t) \cdot P^{b}({X}_t|Y_{t-1},{u}_t).
\end{align}
With Assumption \ref{assm_tnvrtble} the matrix $P^{b}({X}_t|Y_{t-1},{u}_t)$ is invertible, which results in the following equation:
\begin{align}
    \nonumber&P^{b}({X}_{t+1},{y}_t|{X}_t,{u}_t)=\\
    &\hspace{30pt}P^b({X}_{t+1},{y}_t |{Y}_{t-1},{u}_t)\cdot P^{b}({X}_t|Y_{t-1},{u}_t)^{-1},
    \label{after_proxy_y_t-1}
\end{align}
where, under the assumption, we have incorporated the first proxy $y_{t-1}$ for $x_t$.
To incorporate the second proxy, we consider the following matrix expression:
\begin{align}
P^b({Y}_t|{Y}_{t-1},{u}_t )=P^b({Y}_t|X_t,{Y}_{t-1},{u}_t )\cdot P^b({X}_t|{Y}_{t-1},{u}_t ),
\end{align}
where we have again used the law of total probability and Bayes theorem to introduce $X_t$ into the expression.
% Since $y_t \perp\!\!\!\!\perp y_{t-1}| (x_t,u_t)$, by the relation between any state and its observation, we can reduce this expression to :
% \begin{align}
% P^b({Y}_t|{Y}_{t-1},{u}_t )=P^b({Y}_t|X_t,{u}_t )\cdot P^b({X}_t|{Y}_{t-1},{u}_t ),
% \end{align}
% so we modify the expression to give us:
With Assumption \ref{assm_tnvrtble} and condition $y_t \perp\!\!\!\!\perp y_{t-1}| (x_t,u_t)$, we obtain
\begin{align}
    P^{b}({X}_t|Y_{t-1},{u}_t)= P^b({Y}_t|{X}_{t},{u}_t )^{-1} \cdot P^b({Y}_t|{Y}_{t-1},{u}_t ).
\end{align}
We substitute the expression for $P^{b}({X}_t|Y_{t-1},{u}_t)$ back into \eqref{after_proxy_y_t-1} to incorporate the second proxy $y_t$ to achieve
\begin{align}
    \nonumber&P^{b}({X}_{t+1},{y}_t|{X}_t,{u}_t)=P^b({X}_{t+1},{y}_t |{Y}_{t-1},{u}_t)\\
    &\hspace{50pt}\cdot P^b({Y}_t|{Y}_{t-1},{u}_t )^{-1}\cdot P^b({Y}_t|{X}_{t},{u}_t ).
    \label{after_proxy_y_t1}
\end{align}
Similarly, we can recast the vector $p^{\boldsymbol{b}}(r^s_t,y_{t+1}|X_{t+1},u_{t+1})$ as follows:
\begin{align}
    \nonumber&P^{b}(r^s_t,y_{t+1}|{X}_{t+1},{u}_{t+1})= P^b(r^s_t,y_{t+1}|{Y}_{t-1},{u}_t)\\
    &\hspace{50pt}\cdot P^b({Y}_t|{Y}_{t-1},{u}_t )^{-1}\cdot P^b({Y}_t|{X}_{t},{u}_t ).
    \label{after_proxy_y_t2}
\end{align}
In \eqref{p^e(r^s_i)_vector}, we analyze the product of two such terms at indexes $k$ and $k+1$ :
\begin{align}
    \nonumber&P^{b}({X}_{k+1},{y}_k|{X}_k,{u}_k) \cdot P^{b}({X}_{k},{y}_{k-1}|{X}_{k-1},{u}_{k-1})=\\
    \nonumber& P^b({X}_{k+1},{y}_k | {Y}_{k-1},{u}_k)\cdot P^b({Y}_k| {Y}_{k-1},{u}_k ) ^{-1} \\
    \nonumber&\hspace{10pt}\cdot P^b({Y}_k| {X}_{k},{u}_k )\cdot P^b({X}_{k},{y}_{k-1} | {Y}_{k-2},{u}_{k-1})\\
    &\hspace{30pt}\cdot P^b({Y}_{k-1}| {Y}_{k-2},{u}_{k-1} ) ^{-1} \cdot P^b({Y}_{k-1}| {X}_{k-1},{u}_{k-1} ).
\end{align}

In this expression, we segregate the matrix product $P^b({Y}_k| {X}_{k},{u}_k )\cdot P^b({X}_{k},{y}_{k-1} | {Y}_{k-2},{u}_{k-1})$.
We consider the following matrix expression:
\begin{align}
&\nonumber P^{b}({Y}_{k},y_{k-1}|Y_{k-2},u_{k-1})\\
&=P^b({Y}_k| {X}_{k},y_{k-1},Y_{k-2},{u}_k )\cdot P^b({X}_{k},{y}_{k-1} | {Y}_{k-2},{u}_{k-1}),\\
&=P^b({Y}_k| {X}_{k},{u}_k )\cdot P^b({X}_{k},{y}_{k-1} | {Y}_{k-2},{u}_{k-1}),
\end{align}
where, in the second equation, we use the state and observation relation to get $y_k \perp\!\!\!\!\perp (y_{k-1},y_{k-2})| (x_k,u_k) $ and reduce the expression to
\begin{align}
\nonumber &P^{b}({Y}_{k},y_{k-1}|Y_{k-2},u_{k-1})
\\
&=P^b({Y}_k| {X}_{k},{u}_k )\cdot P^b({X}_{k},{y}_{k-1} | {Y}_{k-2},{u}_{k-1}).
\end{align}
As a result, the segregated term $P^b({Y}_k| {X}_{k},{u}_k )\cdot P^b({F}_{k},{y}_{k-1} | {Y}_{k-2},{u}_{k-1})$ can now be re-written without the unobserved state as 
\begin{align}
    \nonumber P^b({Y}_k| {X}_{k},{u}_k )\cdot P^b({X}_{k},{y}_{k-1} | {Y}_{k-2},{u}_{k-1})\\
    =P^{b}({Y}_{k},y_{k-1}|Y_{k-2},u_{k-1}),\label{weight_matrices}
\end{align}
where the right-hand side of this equation is the weight matrix $W_k(\tau^o_k)$. 
This shows that the consecutive multiplication of terms at indexes $k$ and $k+1$ is equivalent to the multiplication of weight matrices.
Substituting the weight matrix in \eqref{weight_matrices} into the expression for reward distribution in \eqref{p^e(r^s_i)_vector} will give us the result presented in Theorem \ref{Theorem_reward_distrib}.

\section*{Appendix B - Discussions}

\noindent \textbf{Finite Belief Space:}
In the context of human-AI recommendation systems, the AI communicates with a human who serves as the receiver. Typically, humans interpret and process information from these recommendations or signals in categorical terms rather than as nuanced probability distributions. For instance, an individual’s belief about the likelihood of rain on a given day may range from low to medium to high, reflecting their level of confidence. Additionally, due to computational constraints, a receiver may simplify their belief space by utilizing a finite set of categories, thus reducing the complexity involved in updating their beliefs to make self-interested decisions.

\noindent \textbf{Information vector:}
In this work, we have considered that the receiver's reward realizations are observable by the sender. 
Consequently, the information vector includes $(\rho^r_{0:i-1})$.
Based on this definition, we define the corresponding notion of state for the POMDP.
However, in settings where the receiver's reward need not be observable to the sender, the information vector must be redefined to exclude this information.
In such cases, the notion of state can be adapted accordingly, such that the Markovian dynamics of the system is preserved.

\vspace{15pt}

\noindent \textbf{Invertibility of statistically computed matrices:} \hspace{30pt}
The assumption that the matrices $P^{\boldsymbol{b}}(Y_t \mid Y_{t-1},u_t)$ and $P^{\boldsymbol{b}}(X_t \mid Y_{t-1}, u_t)$ are invertible may not hold when they are computed from data.
The quality of available offline data plays a crucial role to determine the validity of this assumption.
However, as shown in \cite{tennenholtz2020off}, the invertibility of these matrices holds without explicitly accounting for their empirical estimation.
This motivates the consideration of a set of matrices that are both invertible and sufficiently close to the empirical estimates.
In this case, a robust estimate of off-policy evaluation could be useful, which can serve as an avenue for future research.

\end{document}